\documentclass{article}

     \usepackage[preprint]{neurips_2024}
\usepackage[utf8]{inputenc} 
\usepackage[T1]{fontenc}    
\usepackage{hyperref}       
\usepackage{url}            
\usepackage{booktabs}       
\usepackage{amsfonts}       
\usepackage{nicefrac}       
\usepackage{microtype}      
\usepackage{xcolor}         

\usepackage{amsmath}
\usepackage{amssymb}
\usepackage{amsthm}
\usepackage{bm}
\usepackage[thinc]{esdiff}

\usepackage{algorithm}
\usepackage{algpseudocode}
\usepackage{bbm, dsfont}

\usepackage{graphicx}

\newcommand{\mc}[1]{\mathcal{#1}}
\newcommand{\argmin}{\mathop{\mathrm{arg\,min}}}
\newcommand{\argmax}{\mathop{\mathrm{arg\,max}}}
\newcommand{\mbf}[1]{\mathbf{#1}}
\newcommand{\mb}[1]{\mathbb{#1}}

\newcommand{\hst}{\hspace{0.2cm}}

\newtheorem{theorem}{Theorem}

\newtheorem{definition}{Definition}

\title{Preference-Optimized Pareto Set Learning for Blackbox Optimization}

\author{%
  Zhang Haishan $^*$ \\
  Department of Computational Biology and Medical Sciences\\
  The University of Tokyo\\
       \And
      Diptesh Das $^{*\dagger}$\\
      Department of Computational Biology and Medical Sciences\\
  The University of Tokyo\\
       \AND
      Koji Tsuda$^\dagger$ \\
      Department of Computational Biology and Medical Sciences\\
  The University of Tokyo\\
}

\begin{document}

\maketitle

\def\thefootnote{*}\footnotetext{Equal contribution}
\def\thefootnote{$\dagger$}\footnotetext{Corresponding author~(diptesh.das@edu.k.u-tokyo.ac.jp, tsuda@k.u-tokyo.ac.jp)}

\begin{abstract}
Multi-Objective Optimization (MOO) is an important problem in real-world applications. However, for a non-trivial problem, no single solution exists that can optimize all the objectives simultaneously. In a typical MOO problem, the goal is to find a set of optimum solutions (Pareto set) that trades off the preferences among objectives. Scalarization in MOO is a well-established method for finding a finite set approximation of the whole Pareto set (PS). However, in real-world experimental design scenarios, it's beneficial to obtain the whole PS for flexible exploration of the design space. Recently Pareto set learning (PSL) has been introduced to approximate the whole PS. 
PSL involves creating a manifold representing the Pareto front of a multi-objective optimization problem. A naive approach includes finding discrete points on the Pareto front through randomly generated preference vectors and connecting them by regression. However, this approach is computationally expensive and leads to a poor PS approximation. We propose to optimize the preference points to be distributed evenly on the Pareto front. Our formulation leads to a bilevel optimization problem that can be solved by e.g. differentiable cross-entropy methods. We demonstrated the efficacy of our method for complex and difficult black-box MOO problems using both synthetic and real-world benchmark data.\looseness=-1    
\end{abstract}
\section{Introduction}
Many real-world applications demand the simultaneous optimization of multiple objectives. An efficient and practically useful method to find optimal solutions to a multi-objective optimization (MOO) problem is an active area of research in many real-world experimental design problems and this became more prevalent in the era of robot-driven autonomous design~\citep{abolhasani2023rise,macleod2020self,langner2020beyond,christensen2021data}. Most real-world applications are blackbox in nature where the true analytical forms of the objective functions are not known in advance and this makes solving the real-world MOO problem more difficult. Hence, we often need to rely on blackbox optimization (BO) where surrogate models are built in an online fashion while keeping the laborious and expensive wet-lab experiments or costly computer simulations to the minimum. Therefore, given a computational and monetary budget, the challenge of a BO is to find an efficient and useful method that can find the optimal solutions as quickly as possible~\citep{ojih2022efficiently,terayama2020pushing,seifrid2022routescore,ueno2016combo}. However, in a typical MOO problem, there exists no single solution that can minimize all the objectives at the same time and the goal is to find a set of optimum solutions called Pareto set (or its image called Pareto font) that trades off the preferences among objectives. 
Different MOO algorithms have been developed over years~\citep{deb2002fast,zhang2007moea,shahriari2015taking,sener2018multi,lin2019pareto,mahapatra2020multi,liu2021profiling,zhao2023hypervolume}. Most of the algorithms focused on generating a single solution, or a finite set of Pareto set solutions.  
However, under the mild conditions for a continuous optimization problem with $m$ objectives, the Pareto set lies on a $(m-1)$ dimensional continuous manifold. Therefore, finding a finite set approximation of the entire Pareto set is often not so useful from a decision maker's perspective as it may not contain the solutions of desired preferences.\looseness=-1 

Other issues are scalability and flexibility, which are important in BO, specifically in robot-driven autonomous design. It is not known beforehand which preference vectors will lead to Pareto solutions, and the problem becomes more difficult in BO where the true analytical forms of the objectives are not known. It is computationally intractable to run a learning algorithm for every new preference vector (which is possibly infinite), and it does not allow flexible explorations on the design space unless we have access to all Pareto solutions. Hence, it is desirable to train a set model that can be queried at inference time for any new preference vector and allows flexible iterative exploration of the design space~\citep{terayama2020pushing,kaiya2022understanding}. 
\begin{figure}[t]
\centering
\includegraphics[width=\linewidth]{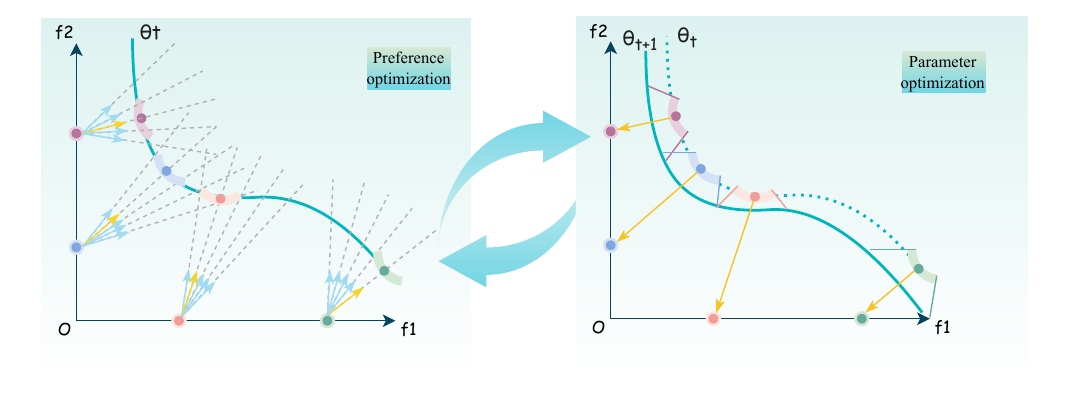}
\caption{An intuitive illustration of the proposed PO-PSL. A uniformly distributed set of preference vectors (left) and the corresponding set model parameters (right) are dynamically selected in alteration. The dots on the orthogonal axes represent the coordinates of a set of reference points.}
\label{diagram}
\end{figure}
Recently, there have been a few studies to model the entire Pareto set/font of a MOO, specifically using deep neural network models~\citep{lin2020controllable,pmlr-v119-ma20a,navon2021learning,ruchte2021scalable,lin2022pareto,chen2022multi,dimitriadis2023pareto}. It has gained significant attention in the MOO community and applied to real-world applications such as drug design~\citep{jain2023multi}, multitask image classification~\citep{raychaudhuri2022controllable,wang2022computing}, multiobjective neural combinatorial optimization~\citep{lin2022-PSL-MOCO}. Recently a data-augmented PSL strategy~\citep{lu2024you} and a hypervolume-based PSL~\citep{zhang2024hypervolume}  have been proposed. Similar to PSL, the conditional model has been recently used to generate a set of diverse solutions in reinforcement learning~\citep{yang2019generalized,abels2019dynamic,basaklar2022pd,hong2022confidence,swazinna2022user}.

However, most of these existing PSL methods can recover only a part of the entire Pareto set/font and struggle to converge to the true Pareto set/font. This is specifically prevalent for complex problems where the Pareto font is non-convex, degenerate, irregular (i.e. disconnected), or problems having more than two objectives. These methods are difficult to train, cannot recover solutions at the boundary, and are not so accurate. This is mainly due to their inefficient training strategy or problem formulation.
To address the limitations of existing PSL methods, we propose a preference-optimized Pareto set learning (PO-PSL) method that is computationally more efficient and also more accurate. 
%
%
We combined a regression model of PSL with dynamic preference vector adaption using optimization-based modeling. This leads to a bilevel optimization algorithm that learns the optimal set model and the optimal preference vector in an end-to-end fashion. Thus, the proposed method can learn a uniformly distributed set of preference vectors which is further connected by a regression model of PSL to better approximate the continuous manifold of the Pareto set.
%
Therefore, the main contributions of this paper are summarized below.\looseness=-1
\begin{itemize}
\item We formulated the PSL problem as a bi-level optimization problem using optimization-based modeling. It accelerates the learning process of the Pareto set model and leads to faster convergence to the true Pareto font.
\item We provided a preference-optimized PSL algorithm that can learn the entire Pareto font while optimizing the preference vector in an end-to-end fashion. 
\item
The proposed method can provide a better approximation of the true Pareto font while the existing PSL algorithms can only learn a part of it and struggle in the case of non-convex, irregular, and degenerate Pareto fonts. 
\item We conducted comprehensive experiments using both synthetic and real-world benchmark data to demonstrate the effectiveness of our method.  
\end{itemize}
\paragraph{Code.} The source code is available at \url{https://github.com/tsudalab/POPSL}.
\section{Problem Statement}
\paragraph{Blackbox optimization (BO).} This is an inverse problem where we search for the optimal solutions in the design space ($\mc{X}$) corresponding to the desired objectives in   the objective space ($\mc{Y}$). Examples of BO are material design~\citep{hase2018chimera,terayama2020pushing}, peptide design~\citep{murakami2023design,zhu2024sample}, etc. where we often don't know the true analytical form of the objective functions and their experimental evaluations are also quite expensive. Hence, given a computational and resource budget, the goal of a typical BO problem is to iteratively build a surrogate model to approximate the true objectives and use an acquisition function to search for optimal solutions. Several single objective~\citep{brochu2010tutorial,shahriari2015taking,frazier2018tutorial} and multi-objective BO have been developed over decades~\citep{khan2002multi,laumanns2002bayesian,knowles2006parego,zhang2009expensive,paria2020flexible} to find a single or a representative set of solutions that reflect the choice of a decision maker or meet the preferred design criteria of a problem.
\paragraph{Multi-objective optimization (MOO).}
In MOO we aim to solve the following optimization.
\begin{equation}\label{eq:moo}
 x^\ast = \argmin_{x \in \mc{X}} \mbf{f}(x) =  \big( f_1(x), f_2(x), \ldots, f_m(x) \big),
\end{equation}
where $x^\ast$ is a solution in the decision space $\mc{X} \in \mb{R}^n$ and $ \mbf{f}:\mc{X} \rightarrow \mb{R}^m$ is an $m$ dimensional vector-valued objective function. However, for a non-trivial problem, there exists no single solution that can minimize all the objectives $\{f_i\}_{i=1}^m$ simultaneously. Therefore, in a typical MOO problem, we are interested in a set of solutions $\mc{P}$, called the Pareto set. Each Pareto solution $x^\ast \in \mc{P}$ represents a different optimal trade-off among the objectives for a MOO problem~(\ref{eq:moo}) and the number of Pareto solutions can be infinite (i.e., $|\mc{P}| = \infty$). A few definitions in the context of MOO are given below.
\begin{definition}[Pareto Dominance] A solution $x$ in the design space is said to Pareto dominate another solution $x^\prime$, i.e. $x\prec x^\prime$, if \hst $\forall i \in [m], f_i(x) \leq f_i(x^\prime)$ and $\exists j \in [m]: f_j(x) < f_j(x^\prime)$.
\end{definition}
\begin{definition}[Pareto Optimality]
A solution $x^\ast$ in the design space is said Pareto optimal if there exists no solution $x$ that Pareto dominates $x^\ast$.
\end{definition}
\begin{definition}[Pareto Set and Pareto Font]
The set of all Pareto optimal solutions in the design space is called the Pareto set (PS) and is defined as $\mc{P}:=\{x^\ast | \nexists x: x\prec x^\ast, \forall x, x^\ast \in \mc{X} \in \mb{R}^n\}$. The corresponding image set in the objective space is called the Pareto font (PF) and is defined as $\mc{F}:=\{f(x^\ast) \in \mb{R}^m | x^\ast \in \mc{P}\}$. 
\end{definition}
%
%
%
%
\paragraph{Pareto estimation}
 Under mild conditions, it is known that for an $m$-dimensional MOO problem both the Pareto set $\mc{P}$ and the Pareto font $\mc{F}$ lie on $(m-1)$ dimensional manifold in design space ($\mc{X} \in \mb{R}^n$) and objective space ($\mc{Y} \in \mc{R}^m$) respectively~\citep{hillermeier2001generalized,wang2012regularity}. 
Moreover, in real-world applications, we often don't know the true analytical form of the objective functions, and its experimental evaluations are also quite expensive.
Therefore, navigating over an unknown manifold of Pareto set/font for a complex and difficult blackbox MOO is a challenging problem.
Hence, for real-world complex and expensive MOO problems, it is desirable to exploit the structure of the Pareto set or Pareto font as it can lead to faster convergence and or preferred exploration both in the design and objective spaces~\citep{lin2022pareto}.\looseness=-1 

Several attempts have been made to approximate the Pareto set or Pareto font over the decades. Earlier works in this direction focussed on population-based evolutionary multi-objective optimizations such as NSGA-II~\citep{996017}, R-NSGA-II~\citep{deb2006reference}, MOEA/D~\citep{zhang2007moea}, and others. 
 These population-based approaches can only provide a finite set approximation of the true Pareto set/font and are not scalable to a large deep learning-based MOO with millions of parameters.
 Therefore, gradient-based algorithms have been proposed to solve large scale deep MOO~\citep{desideri2012multiple,wang2017hypervolume,sener2018multi,lin2019pareto,mahapatra2020multi,mahapatra2021exact,liu2021stochastic,liu2021profiling}.
However, most of these gradient-based methods need to train multiple neural network models and cannot learn the continuous manifold of Pareto set/font which is crucial for flexible exploration, such as in BO.  
 %
 %
 %
\paragraph{Learning-based MOO.} Recently, there has been a growing interest among the machine learning community to model the continuous manifold of the Pareto set or Pareto font.
\citet{navon2021learning} proposed to learn a preference-conditioned hypernetwork in the context of deep learning-based MOO, such as deep multi-task learning where the Pareto set is the optimal neural network parameters ($\theta^\ast \in \Theta$). Once, the model is trained, it can generate any Pareto optimal neural network parameters ($\theta^\ast$) for any preference vector ($w\sim \mc{W}$).~\citet{pmlr-v119-ma20a} proposed an efficient algorithm to construct a locally continuous Pareto set/font using first-order expansion. On the other hand,~\citet{lin2022pareto} proposed a pioneering work to model the continuous manifold of Pareto set for expensive BO (similar to ours), where the author proposed to learn a parameterized deep neural network-based regression model to generate Pareto optimal solution in the design space ($x^\ast \in \mc{X}$) corresponding to any preference vector ($w\sim \mc{W}$). Another work~\citep{chen2022multi} which is similar in spirit but fundamentally different to that of our method, proposed an adaptive preference vector-based deep MOO, where the goal is to generate uniformly distributed (but discrete) Pareto optimal neural network parameters ($\theta^\ast$) conditioned on the preference vector ($w$). Moreover, the adaptive strategy proposed in~\citet{chen2022multi} can only provide a discrete approximation of the continuous manifold of neural network parameters ($\theta^\ast$), whereas our proposed adaptive strategy provides a continuous approximation of the true manifold of design parameters ($x^\ast$).\looseness=-1     

In this paper, we are interested in inverse problem ($\mc{Y} \mapsto \mc{X}$) same as in ~\citet{lin2022pareto}, whereby having direct access to $\mc{W}\mapsto \mc{X}$ mapping, a decision maker can readily generate preferred solutions in the design space ($\mc{X}$).
\paragraph{Advantages of learning-based MOO.}
There are two main advantages of such learning-based approaches: First. Once the regression model is learned, one can readily (in real-time) and freely explore the manifold of Pareto set/font for different objective preferences ($w$) without rerunning the underlying algorithm for infinite preference vectors, thus leading to a significant reduction in computational overhead. Second. The structural information based on the domain knowledge can be incorporated as constraints into the objective of the regression model for faster and more accurate convergence to the true Pareto set/font. For details please see the proposed method section.
%
%
%
\paragraph{Scalarization.} A standard approach to finding a finite set approximation of $\mc{P}$ is called the scalarization method. A scalarization method employs a function $\Omega(\cdot): \mb{R}^m \rightarrow \mb{R}$ and provides a natural connection between a set of trade-off preferences $\mc{W}=\{w\in \mb{R}_+^m | \sum_{i=1}^m w_i = 1\}$ among $m$ objectives to the corresponding Pareto set $\mc{P}$. A simple scalarization function is the weighted-sum scalarization:\looseness=-1
$$
\Omega_{ws} (x,w) = \sum_{i=1}^m w_i f_i(x).
$$
Therefore, for a given preference vector $w^\prime \in \mc{W}$, one can find the corresponding Pareto optimal solution $x({w^\prime})$ by solving the following optimization problem:
\begin{equation}\label{moo-ws}
x^\ast({w^\prime}) = \argmin_{x \in \mc{X}} \Omega_{ws} (x,w^\prime).
\end{equation}
By repeatedly solving (\ref{moo-ws}) for a predefined set of trade-off preference vectors, one can obtain a set of Pareto optimal solutions. %
However, the solution obtained by solving (\ref{moo-ws}) corresponds to the specific preference vector $w^\prime$ and does not provide any information about how the solution would change with a slight change in $w^\prime$. 
Moreover, the scalarization-based method can only provide a finite set approximation of the whole Pareto set $\mc{P}$,
whereas finding the whole Pareto set can be useful as a user can explore the corresponding Pareto font of the objective space based on their preference. 
\paragraph{Pareto set learning (PSL).}
In PSL~\citep{lin2022pareto}, instead of solving (\ref{moo-ws}) directly, we learn a mapping $\theta \in \Theta:$
\begin{equation}\label{PSL}
w \mapsto h_{\theta}(w):=x(w),
\end{equation}
and the original problem (\ref{moo-ws}) of scalarized MOO is converted to the following optimization :
\begin{equation}\label{eq:obj-psl}
\theta^\ast = \argmin_{\theta \in \Theta} \Omega \big(x(w)=h_\theta (w), w \big), \quad \forall w \in \mc{W}.
\end{equation}
Then, one can apply the standard gradient update rule to learn the mapping parameter $\theta:$\looseness=-1 
$$
\theta_{t+1} = \theta_t -\eta \sum_{w \sim \mc{W}} \nabla_\theta \Omega(\cdot, w),
$$
where $\Omega$ is a differentiable scalarization function,  $\eta$ is the learning rate and $t \in [T]$ is any arbitrary iteration. 
Once the mapping is properly learned, one can write
\begin{equation}\label{eq:learned-PS}
x^\ast(w) = h_{\theta^\ast}(w), \quad \forall w \in \mc{W},
\end{equation}
where $\theta^\ast$ is the optimum mapping parameter of the learned Pareto set.
Such a mapping enables us to model the manifold of the Pareto set instead of a finite set approximation. Therefore, with such a model a decision-maker can readily explore any trade-off area in the learned Pareto set (hence, in the Pareto front) by adjusting the preferences ($w$) among objectives for flexible decision-making.
\paragraph{Limitations of existing PSL methods.}
However, existing methods of PSL naively leverage an arbitrary (unevenly distributed) set of discrete points on the Pareto front via randomly generated preference vectors and connect them by regression~\citep{lin2022pareto}. This leads to a poor approximation of the manifold of the entire Pareto font which is readily evident for MOO with more than two objectives or complex MOO with non-convex or irregular Pareto font. Another problem of the existing PSL method is the inefficient training strategy which leads to slow convergence and poor approximation to the true manifold. 
The optimal solution set for any scalarization is unknown and we need to optimize all solutions over infinite preferences $(|\mc{W}| =\infty)$ to produce a good fit to the true Pareto set/font. This becomes more difficult in the context of BO where the true analytical forms of the objectives are unknown.
\citet{lin2022pareto} proposed to use Monte Carlo sampling $(\{w_1, \ldots, w_K\} \sim \mc{W}$, where $K$ is finite) and applied iterative gradient update rule:
$$
\theta_{t+1} = \theta_t -\eta \sum_{i =1}^K \nabla_\theta \Omega(\cdot, w_i).
$$
Such a direct or full fitting (Monte Carlo sampling) strategy can be highly inefficient due to prohibitive time and sample complexity for complex learning problems, such as for expensive multi-objective optimization (EMO) problems. 
%
\paragraph{Optimization-based modeling.}Recently, an alternative approach called optimization-based modeling or end-end-learning (E2E) has gained significant interest in the machine learning community for solving complex problem~\citep{gould2016differentiating,domke2012generic}.
This is an optimization technique where one can incorporate domain knowledge or specialized operations into an E2E machine learning pipeline typically in the form of a parametrized  \textit{arg min} operation to make the learning faster and accurate~\citep{johnson2016composing,amos2018differentiable,belanger2017end,agrawal2020learning}.
For example, in an E2E energy-based learning, given $n$ labeled training data points, $\mc{D}_n= \{x_i, y_i\}_{i=1}^n$, a parameterized energy function $E(x, y; \theta)$, and a differentiable loss function $L(x, y; \theta)$;  the energy ($E$) minimization and model ($\theta$) optimization are separated into two steps, where the solution of energy minimization (or a truncated energy minimization i.e., energy minimization is possibly incomplete) is used to train the model parameter which leads to faster convergence and better accuracy~\citep{belanger2017end,domke2012generic}. Basically, this is a bilevel optimization framework where the inner level objective determines the optimal prediction $y^\ast$ (or sub-optimal $\hat y$) for a given model parameter $\theta$, and at the outer level the goal is to identify the optimum model parameter $\theta^\ast$ which corresponds to $y^\ast$ (or $\hat y$), thus collectively minimizing $L$.\looseness=-1 
\section{Proposed Method}
\paragraph{Preference-optimized Pareto set learning (PO-PSL).}
Following the optimization-based modeling, we propose a bilevel optimization for preference-optimized Pareto set learning and solve (\ref{eq:obj-psl}) in an iterative manner: 
\begin{equation}\label{eq:w_estimate}
w^\ast \leftarrow \argmin_{w} \Omega (x, w; \theta),
\end{equation}  
\begin{equation}\label{eq:theta_estimate}
\theta^\ast \leftarrow \argmin_{\theta} \Omega (x, w; \theta),
\end{equation}
First $w$ is optimized by fixing $\theta$. As a result, we obtain the optimal solution $w^\ast$ which is subsequently used to update $\theta$.
To learn the optimum parameter $\theta^\ast$ one can consider a gradient-based optimization where an iterative update rule for any iteration $t \in [T]$ can be written as\looseness=-1
 \begin{align}\label{eq:update_rule}
 \nonumber
 \theta_{t+1} &= \theta_t -  \eta \nabla_\theta \Omega  \big(x, w; \theta \big),\\
 &= \theta_t -  \eta \nabla_\theta \Psi \big(w^\ast(x;\theta), w \big),
\end{align}
 %
%
 where $\eta \in \mb{R}$ is the learning rate and $\Omega$ is an implicit function of $\Psi$, another differentiable loss which compares the predicted optimum preference $w^\ast(x;\theta)$ to the true preference $w$. 
If it is possible to solve \eqref{eq:w_estimate} and $\frac{dw^\ast}{d\theta}$ exists, then one can apply implicit function theorem (Theorem~\ref{thm:implicit_fun}) to get the gradient update of the mapping parameter $\theta$ and make the optimization procedure end-to-end learnable~\citep{domke2012generic,belanger2017end,belanger2016structured,lecun2006tutorial}. 
 \begin{theorem}\label{thm:implicit_fun}
 Define $w^\ast$ as in \eqref{eq:w_estimate} and $\Omega(\theta) = \Psi(w^\ast (\theta))$. Then, when all the derivatives exist,
 $$
 \nabla_\theta \Omega \vert_{w=w^\ast} = - \frac{\partial^2 \Omega}{\partial \theta \partial {w^\ast}^\top} \Big( \frac{\partial^2 \Omega}{\partial w^\ast \partial {w^\ast}^\top} \Big)^{-1} \frac{\partial \Psi}{\partial w^\ast}.
 $$
 \end{theorem}
 \begin{proof}
 Let $g(w, \theta):=\frac{\partial \Omega(w, \theta)}{\partial w}$. Then
 \begin{align*}
 g(w^\ast, \theta)&=\frac{\partial \Omega(w, \theta)}{\partial w} \Big\vert_{w=w^\ast} = 0, \quad (\text{using \eqref{eq:w_estimate}})\\
 \therefore \frac{\partial g (w^\ast, \theta)}{\partial \theta} &= \frac{\partial g}{\partial w^\ast} \frac{\partial w^\ast}{\partial \theta} + \frac{\partial g}{\partial \theta} = 0,\\
 \implies \frac{\partial w^\ast}{\partial \theta} &= - \frac{\partial g}{\partial \theta}\Big(\frac{\partial g}{\partial w^\ast} \Big)^{-1},\\
 &=-\frac{\partial^2 \Omega}{\partial \theta \partial {w^\ast}^\top} \Big( \frac{\partial^2 \Omega}{\partial w^\ast \partial {w^\ast}^\top} \Big)^{-1}.\\
 & (\text{using vector notation})
\end{align*}
Now,
\begin{align*}
\nabla_\theta \Omega \vert_{w=w^\ast} &= \frac{\partial \Psi}{\partial w^\ast} \frac{\partial w^\ast}{\partial \theta},\\
&= -\frac{\partial^2 \Omega}{\partial \theta \partial {w^\ast}^\top} \Big( \frac{\partial^2 \Omega}{\partial w^\ast \partial {w^\ast}^\top} \Big)^{-1}\frac{\partial \Psi}{\partial w^\ast}.
\end{align*}
 \end{proof}
 Note that in the update rule (\ref{eq:update_rule}), $\Omega(\theta)=\Psi(w^\ast(\theta)),$ i.e., $w^\ast$ is a function of $\theta$ and we need to be able to compute $\frac{d w^\ast}{d\theta}$. However, computing the derivative exactly through the non-convex \textit{arg min} in (\ref{eq:w_estimate}) is challenging both in theory and practice, and several numerical methods have been developed in this context. One direction of studies focussed on gradient unrolling method~\citep{domke2012generic,andrychowicz2016learning,gould2016differentiating,finn2017model,belanger2017end,monga2021algorithm,chen2022learning} that leverages the chain rule of differentiation and unroll the gradient through backpropagation.
Another approach is the use of a zeroth-order gradient estimation method such as the recently proposed~\citep{amos2020differentiable} differentiable cross entropy method (DCEM).
\paragraph{Zeroth-order gradient estimation.}
Zeroth-order gradient estimation method such as cross entropy method (CEM)~\citep{rubinstein1997optimization,de2005tutorial,bharadhwaj2020model,pinneri2021sample,amos2020differentiable} or evolutionary strategy (ES) method~\citep{li2020evolution,hansen2003reducing,he2022adaptive} has been successfully used in many applications.
It estimates the gradient by generating a sequence of samples from the objective function.
For example in CEM, a sequence of samples is generated considering a parameterized prior distribution. The parameter of the distribution is iteratively updated by refitting the sample distribution to the top $k$ samples by solving a maximum-likelihood problem. 
DCEM~\citep{amos2020differentiable} is a differentiable CEM that leverages the Limited Multi-Label Projection (LML) layer~\citep{amos2019limited} to make the non-differentiable top-$k$ operation differentiable. This enables the output of CEM differentiable with respect to the objective parameter and allows CEM to be integrated as a part of the end-to-end machine-learning pipeline.
The zeroth-order method is often preferred over gradient unrolling for large-scale complex optimization problems as it is robust against data noise and generally less susceptible to hyperparameter overfitting~\citep{amos2020differentiable,he2022adaptive}. 
Such limitations of gradient unrolling in the case of structured prediction energy networks (SPENs)~\citep{belanger2016structured,belanger2017end} have been demonstrated in section~4.1.and 5.1. of~\cite{amos2020differentiable}. We used DCEM as a differentiable optimization layer (DOL) in our implementation and it allows us to compute $\diff{w^\ast}{\theta}$  efficiently. 
For details please see the DCEM paper.
Therefore, using DCEM we estimate $w^\ast(x;\theta)$ by solving the following:
 \begin{align*}\label{eq:likelihood-model}
 w^\ast(x;\theta) &:= \mb{E} (g_{\hat{\phi}}(\cdot)),\\
 \hat{\phi} &:= \argmax_\phi \mc{L} (\phi;w, x),\\
 w &\sim g_{\phi} (\cdot),\\
 x&:=h_\theta(w),
 \end{align*}
 where $\mc{L}(\cdot)$ is the likelihood function, $\mb{E}$ is the expectation operator and $g_\phi(\cdot)$ is the prior distribution of $w$, parameterized by $\phi$. 
 \paragraph{Preference optimization.}
 The input to our Pareto set model is the preference vector $w$ located on the simplex \eqref{eq:simplex}.
\begin{equation}
\label{eq:simplex}
    \mc{W} = \{w \in \mathbb{R}^{m} \mid w \geq \mathbf{0}, \mathbf{1}^T w = 1\}.
\end{equation}
Unlike the existing PSL method~\citep{lin2022pareto} which randomly samples a group of preference vectors, we select an optimal set of preference vectors \eqref{eq:w_estimate} as input based on some reference points (a set of random points uniformly distributed along the ideal axis of individual objectives, see Fig.~\ref{diagram}). Given a reference point $z$ and the current learned neural network $h_\theta(\cdot)$, a preference vector $w$ can be seen as the direction starting from the $z$. These directions point towards the current learned Pareto front and intersect with it. 
\paragraph{Loss function.}
We choose $\theta$ based on some loss function $Q$. This loss function can be a reference point-based scalarization, such as weighted-Tchebycheff $(\Omega_{tch})$ or augmented weighted-Tchebycheff scalarization $(\Omega_{aug\_tch})$ or penalized boundary intersection (PBI) $(\Omega_{pbi})$~\citep{zhang2007moea}: 
\begin{align*}
Q = \Omega_{tch} (x, w, z) &:= \max_{i\in [m]}  \{ w_i  |f_i(x) - z_i |\},\\
Q=\Omega_{aug\_tch} (x, w, z) &:= \max_{i\in [m]}  \{ w_i  |f_i(x) - z_i | + \rho \sum_{i\in[m]} w_i f_i(x)\},\\
Q=\Omega_{pbi} (x, w, z) &:= d_1 + \rho d_2, \quad \text{where }
\end{align*}
$$
d_1=\frac{|| (\mathbf{f} (x) - z )^\top w ||}{|| w ||}, \quad d_2=|| \mathbf{f}(x) - (z + d_1 w) ||, \hst \text{and}
$$
$\rho > 0$ is a penalty parameter. 
 To ensure that the model learns the correct Pareto front while maintaining diversity, one can add a penalty term $\zeta$ to the scalarization term $\Omega$:
\begin{equation}\label{eq:penalized_loss}
Q = \Omega + \lambda \zeta,
\end{equation}
where $\lambda>0$ is a tuning parameter.
A common example of the penalty is the cosine similarity~\citep{ruchte2021scalable,chen2022multi} between a preference vector $w$ and the generated Pareto font solution:
$$
\zeta = \frac{w^\top \mathbf{f}(x)}{||w|| \cdot || \mathbf{f}(x) ||}.
$$
We introduced a novel penalty term~\eqref{eq:pbi_const} using a neighborhood of $w^\ast$ that promotes diversity (Fig.~\ref{penalty}). It is motivated by the penalty used in~\cite{chen2022multi}. 
\begin{align}\label{eq:pbi_const}
\nonumber
    \zeta &= \exp(-c_1c_2), \\
    \nonumber
    c_1 &= \frac{\sqrt{2}}{2}- \cos\alpha, \\
    c_2 &= \cos\alpha - 1,\\
     \nonumber
    \cos\alpha &= \frac{(\mathbf{f}(w)-\mathbf{f}(w^\ast))^\top(z-\mathbf{f}(w^\ast))}{\|(\mathbf{f}(w)-\mathbf{f}(w^\ast))\| \|(z-\mathbf{f}(w^\ast))\|}.
\end{align}
\begin{figure}[t]
\centering
\includegraphics[width=0.5\linewidth]{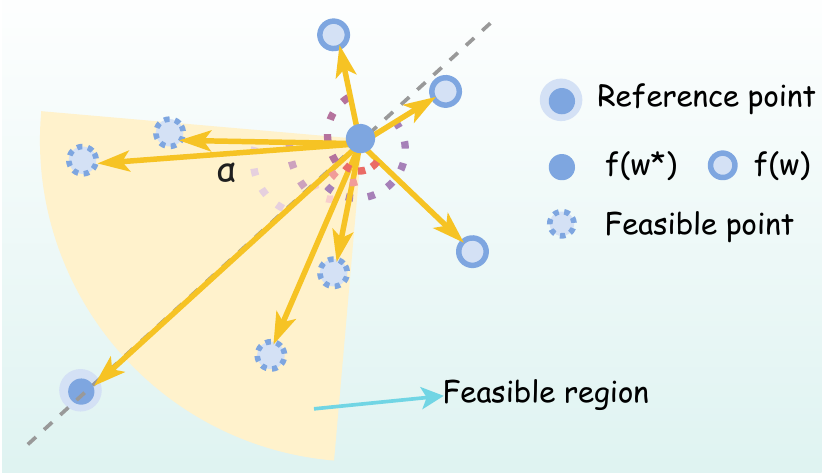}
\caption{An intuitive illustration of the penalty term used in our loss function. This constructs a cone with $\mathbf{f}(w^\ast)$ as the vertex and a half-apex angle of $\frac{\pi}{4}$. We aim for the $w$ in the neighborhood to fall within this cone as much as possible. }
\label{penalty}
\end{figure}
%
 %
%
\paragraph{Training and Algorithm of PO-PSL.}
For any reference point-based differentiable scalarization function our method works as follows. We choose a set of $r$ uniformly random reference points ($Z=\{z_1,\ldots, z_r\}$) which correspond to a set of coordinate points along the ideal axis of individual objectives (Fig.~\ref{diagram}).
 Then the DCEM is trained to predict the preference vector $w^\ast(z)$ corresponding to each reference point $z$. The predicted $w^\ast(z)$  is then used as input to the Pareto set model $h_\theta(w^\ast(z))$ to generate a Pareto optimal solution $x(z)$. The objective value $f(x(z))$ evaluated at $x(z)$ is then compared against the corresponding reference point $z$ using a penalized loss function $Q$~\eqref{eq:penalized_loss}. The algorithm of our proposed PO-PSL method has been provided in Algorithm~\ref{algo-EE-PSL}. To learn the set model $h_\theta(\cdot)$ we used a neural network model. 
Our algorithm starts with a random initialization of the model parameter $\theta=\theta_0$. Then at every iteration $t\in [T]$, $\theta$ is updated based on the averaged error signal of all $r$ reference points which makes the learned Pareto font progress towards to true Pareto font.
 Note that at any iteration $t$, the set model $h_{\theta_t}(\cdot)$ learns an approximate Pareto font manifold 
corresponding to learned parameter $\theta_t$ which allows for flexible exploration of the learned Pareto font along that manifold for any chosen preference $w$. 
%
%
\begin{center}
\begin{minipage}{\linewidth}
\begin{algorithm}[H]
\caption{Preference-optimized Pareto set learning}
\label{algo-EE-PSL}
\begin{algorithmic}[1]
\State \textbf{Input:} Model $x = h_\theta (w)$, $r:$ \# of reference points  
\State Initialize $\theta = \theta_0$
\For {$t=0$ to $T-1$}
\State $w^\ast (z) = \text{DCEM} (\theta_t, z), \quad \forall z \in Z:=\{z_1, \ldots, z_r\}$
\State $\theta_{t+1} = \theta_t -\eta \sum_{z \in Z} \nabla_{\theta} Q(x= h_{\theta}(w^\ast(z)))$ 
\EndFor
\State \textbf{Output:} $h_{\theta_T}(\cdot)$
\end{algorithmic}
\end{algorithm}
\end{minipage}
\end{center}
\section{Results and Discussions}
To evaluate the performance of our proposed PO-PSL, we compared it with some state-of-the-art (SOTA) methods on several complex synthetic and real-world benchmark data. 
 \paragraph{Baseline Algorithms.} The baseline algorithms we chose are the SOTA PSL-MOBO~\citep{lin2022pareto} and the classical MOBO algorithm, DGEMO~\citep{konakovic2020diversity}. The implementations of PSL-MOBO and DGEMO are from their open-source codebases. 
 \paragraph{Benchmarks and Real-world problems.} The algorithms are compared on the following widely-used benchmark data, ZDT3, and DTLZ5, the ground truths of which are provided by the library PYMOO~\citep{pymoo}. The ZDT3 and DTLZ5 have a disconnected and a degenerate Pareto front, respectively. We also test all methods on a real-world rocket injector design (RE5) problem, the ground truth of which is obtained from the DGEMO GitHub page. 
\paragraph{Experimental setting.} 
Gaussian processes are used as the surrogate model for all methods. We repeated our experiments five times and reported the average and standard deviation results of five runs. We used hypervolume difference (HVD) and inverted generational distance (IGD) to evaluate the performance~\citep{pymoo}.
 \paragraph{Batch selection.} 
We use the loss function $Q$ as the selection criterion directly. We first select a group of candidate Pareto solutions with the lowest loss scores for each reference point and then select a batch of samples for evaluation from this group to prevent the selected samples concentrated on the single reference point. 
\paragraph{Hardware platform.} We conduct all the experiments on the Linux GPU server equipped with one V100, 24 cores of CPU (Intel Xeon Gold 6136), 376GB memory, and 32GB GPU memory.  
  \paragraph{PO-PSL setting.} 
The proposed method PO-PSL is implemented in PyTorch and an open-source GitHub implementation is provided. In our experiments, we mainly used the PBI scalarization. Additional results using augmented Chebycheff scalarization are provided in Fig.~\ref{atch}. The mapping network of the Pareto set model contains three MLPs with 256 hidden nodes, and the activation function of the model is the Exponential Linear Unit (ELU) function. 
For DCEM, we sample 1000 preference vectors, and the elite neighborhood number is set to 100. In every iteration, models are trained with batched reference points with a batch size of 8 for each objective, which means there are a total of $8\times 2$ reference points for two objective problems. For every reference point, we select a group of 100 preference vectors. 
The performance of all the PSL-based methods is evaluated on 100 corresponding predictions of the mapping networks with randomly generated 100 preference vectors.\looseness=-1 
%
%
\begin{figure}[h!]
 \centering
 \includegraphics[width=\linewidth]{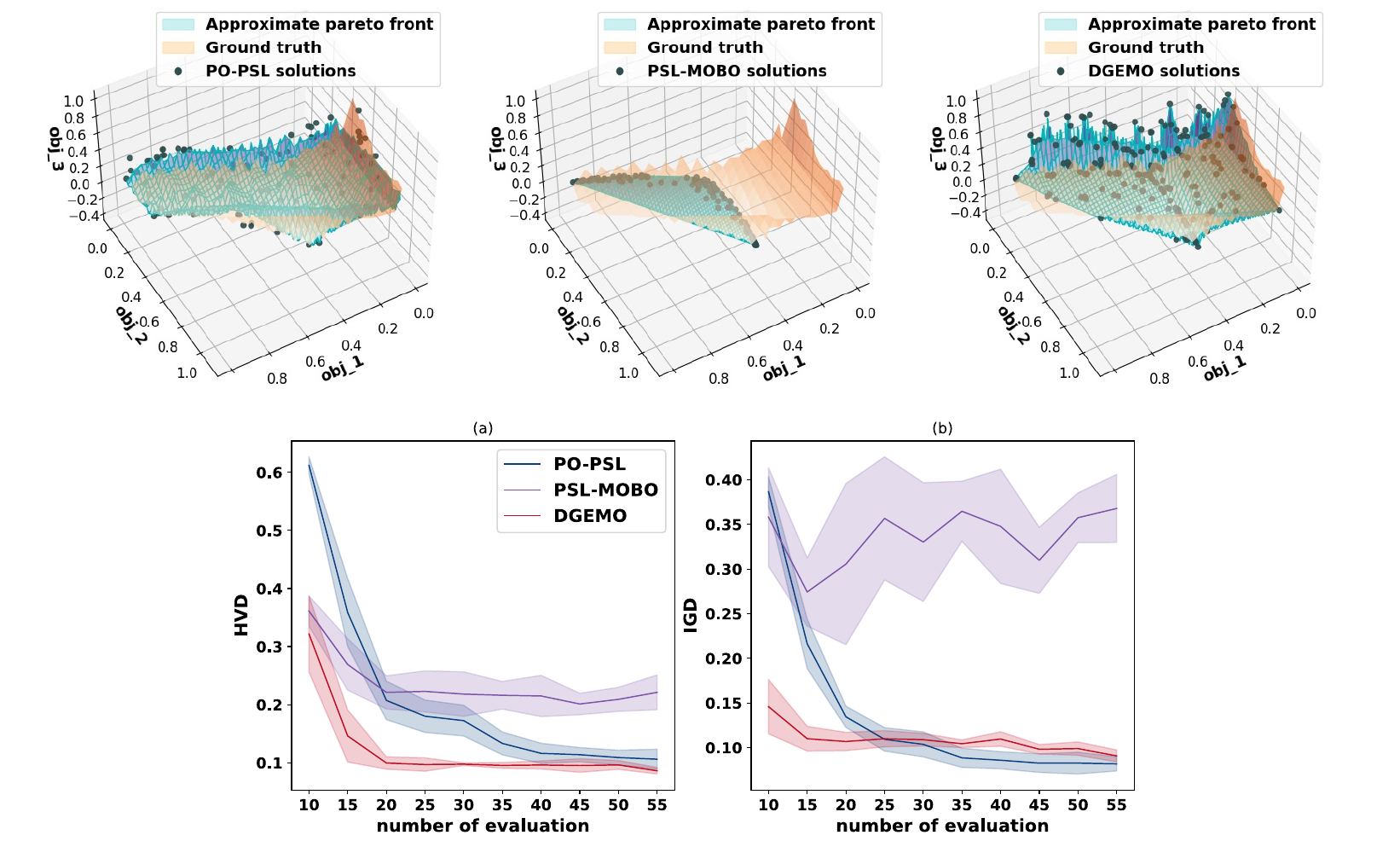}
 \caption{RE5 (continuous PF). First row: PF approximation using different methods. Second row: Sampling efficiency using HVD and IGD. We select 10 random samples as the initial set, and add 5 new samples at each iteration.}
 \label{re-pf}
 \end{figure}
%
%
\begin{figure}[h!]
 \centering
 \includegraphics[width=\linewidth]{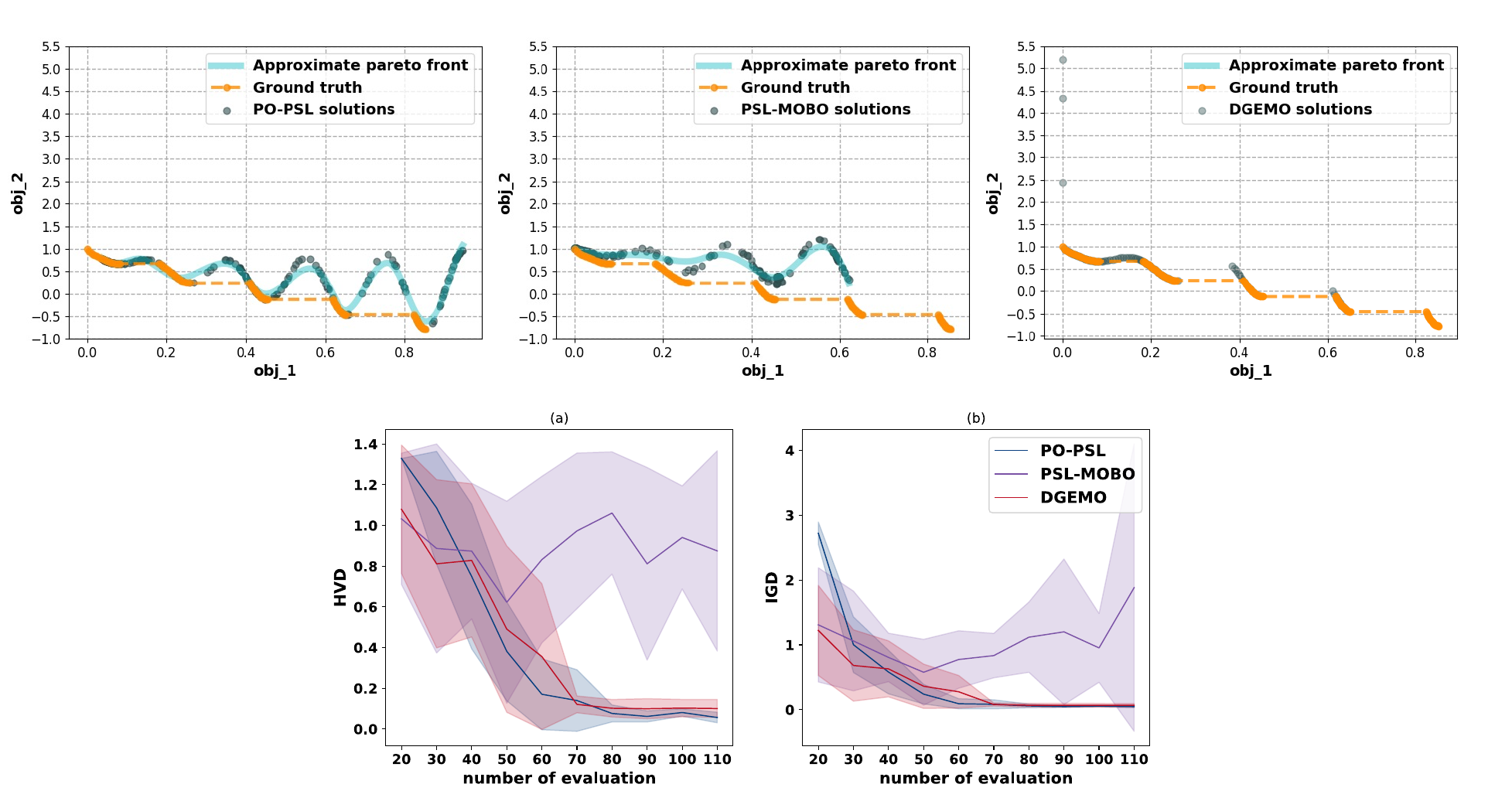}
 \caption{ZDT3 (irregular PF). First row: PF approximation using different methods. Second row: Sampling efficiency using HVD and IGD. We select 20 random samples as the initial set, and add 10 new samples at each iteration.}
 \label{zdt3-pf}
 \end{figure}
%
%
\begin{figure}[h!]
 \centering
 \includegraphics[width=\linewidth]{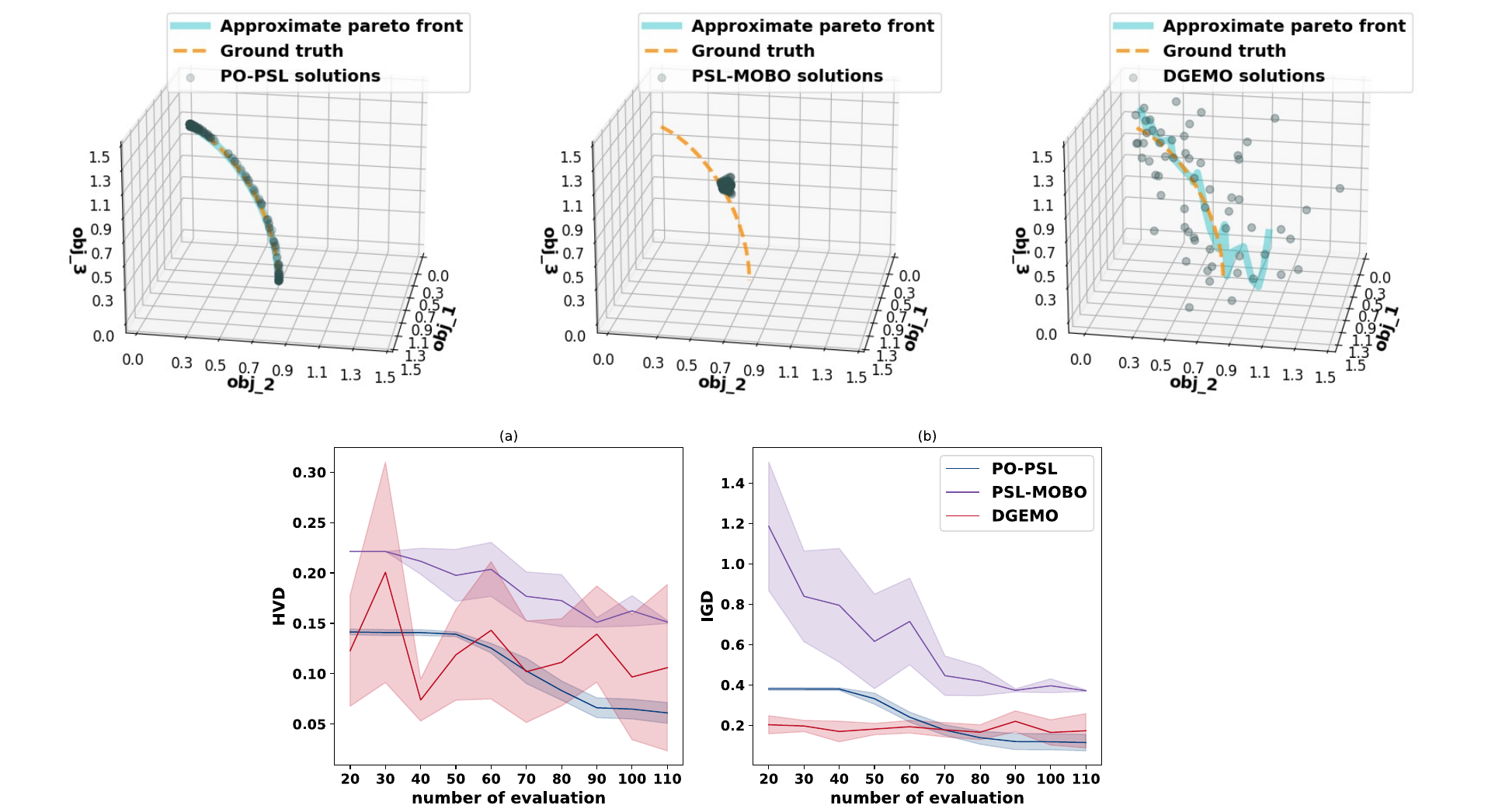}
 \caption{DTLZ5 (degenerate PF). First row: PF approximation using different methods. Second row: Sampling efficiency using HVD and IGD. We select 20 random samples as the initial set, and add 10 new samples at each iteration.}
 \label{dtlz5-pf}
 \end{figure}
%
%
%
%
  \begin{table*}[t]
  \caption{Computational Time (modeling + batch selection) in seconds.}
  \label{table-comp-time}
  \centering
  \begin{tabular}{c c c c c c}
    \hline
    Problem & \#objs. & \#vars.  & PSL-MOBO & PO-PSL  & DGEMO \\
 \hline
    ZDT3 &2 &6 & 0.65 + 2.81 & 0.33 + 0.35 & 3.07 + 0.41  \\
    DTLZ5 &3 &6 & 0.57 + 3.87 & 0.37 + 0.34 & 4.19 + 0.68 \\
    RE5 &3 &4 & 0.79 + 5.49 & 0.35 + 0.33 & 2.64 + 0.17 \\
  \hline
  \end{tabular}
\end{table*}
\paragraph{Sampling and computational efficiencies.} 
PO-PSL shows faster convergence than PSL-MOBO in terms of both HVD and IGD, and produces better PF approximation than others (Fig.~\ref{re-pf},\ref{zdt3-pf},\ref{dtlz5-pf}). It can be observed that PO-PSL is computationally more efficient than others (Table.~\ref{table-comp-time}). 
\paragraph{Ablation study.} Additionally, we conducted two ablation studies: one to verify the impact of different reference point settings on the learned Pareto front~(Fig.~\ref{abl-reference}), and another to verify the effect of the penalty~(Fig.~\ref{abl-penalty}) used in the loss function~\eqref{eq:penalized_loss}. We set different numbers and distributions of reference points. Our results find that different reference point settings can significantly affect the exploration of the Pareto front. Reference points that are evenly distributed on the orthogonal basis contribute to comprehensively exploring the Pareto front, while unevenly distributed points cause the learned Pareto front to concentrate in certain areas. 
We design the penalty item to ensure that the model learns the correct Pareto front while maintaining diversity. We find that the learned Pareto front will contain more outliers without the penalty item, and the half-apex angle can affect the shape of the learned Pareto front. 
%
\begin{figure}[h!]
 \centering
 \includegraphics[width=\linewidth]{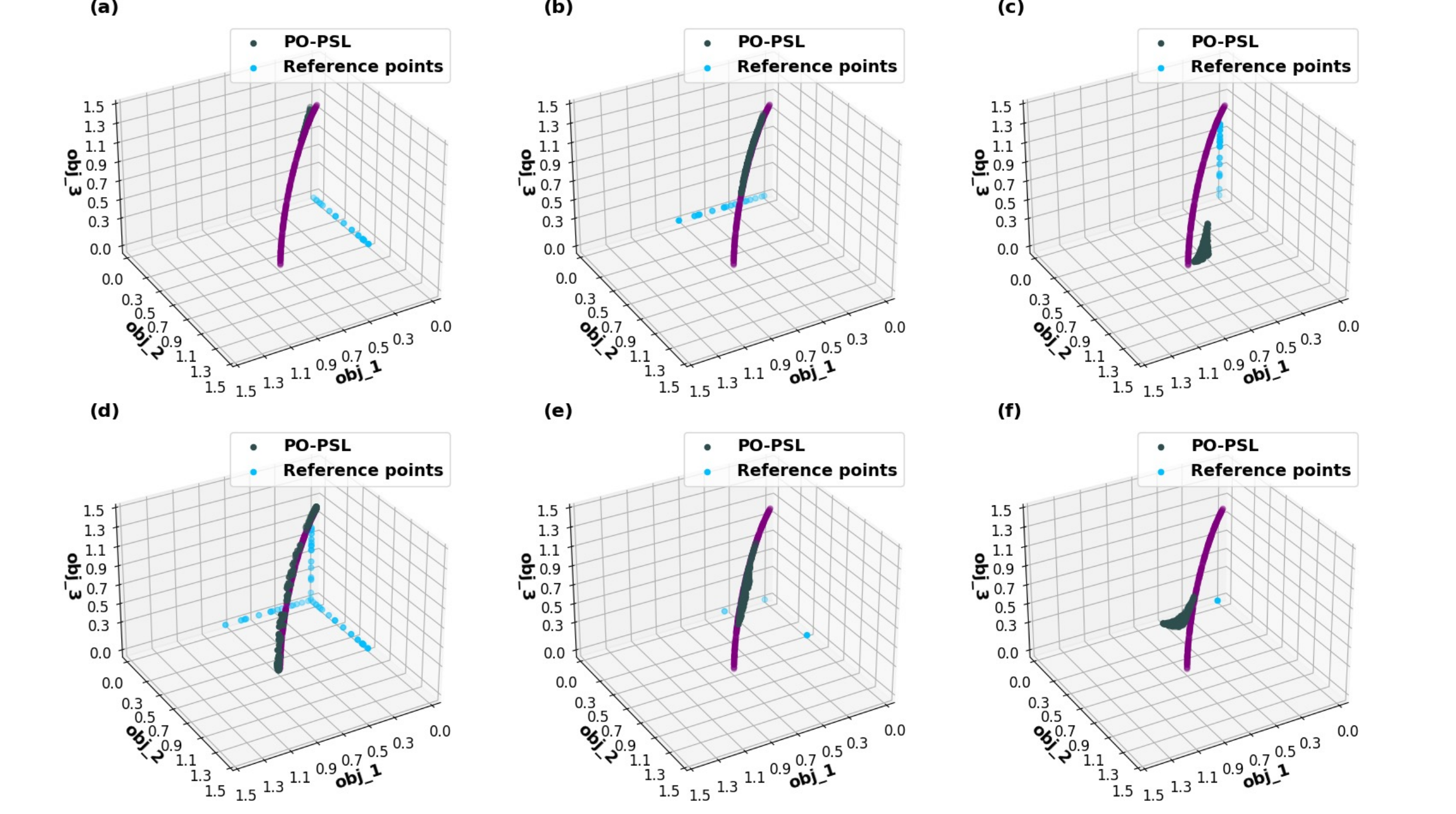}
 \caption{Ablation study for reference points. (a) uses 16 reference points located only on the $obj_1$. (b) uses 16 reference points located only on the $obj_2$ axis. (c) uses 16 reference points located only on the $obj_3$ axis. (d) uses reference points evenly distributed across all three orthogonal bases with $16 \times 3$ points. (e) uses reference points distributed across the three orthogonal bases with 3 points. (f) uses only one reference point.}
 \label{abl-reference}
 \end{figure}
\begin{figure}[h]
 \centering
 \includegraphics[width=\linewidth]{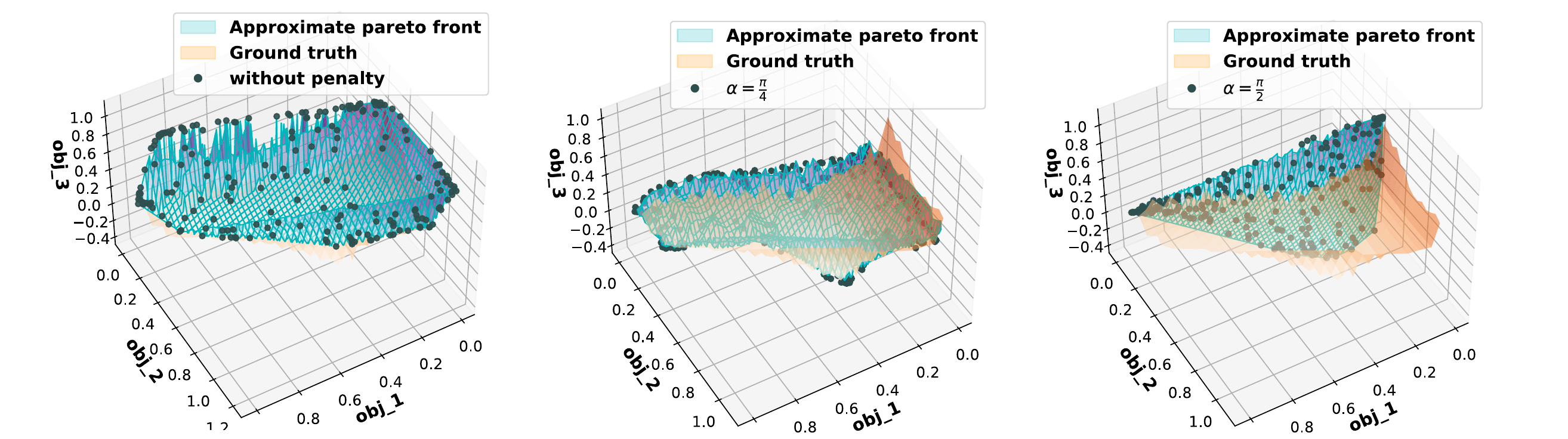}
 \caption{Ablation study for penalty in~\eqref{eq:penalized_loss}. We conduct the ablation study on the dataset of RE5 with three settings: without penalty, setting $\alpha$ as $\frac{\pi}{4}$, and setting $\alpha$ as $\frac{\pi}{2}$.}
 \label{abl-penalty}
 \end{figure}
\paragraph{Results using Tchebycheff scalarization.} Our approach can support different scalarization methods. We also implemented the augmented Tchebycheff scalarization method and the results are shown in Fig.~\ref{atch}. Models are trained with batched reference points with a batch size of 256 in every iteration. We repeated our experiments five times and reported the average and standard deviation results of five runs. Here we considered a multi-layer perceptron (MLP) surrogate model composed of three fully connected layers with 256 hidden nodes.%
\begin{figure}[h]
 \centering
 \includegraphics[width=\linewidth]{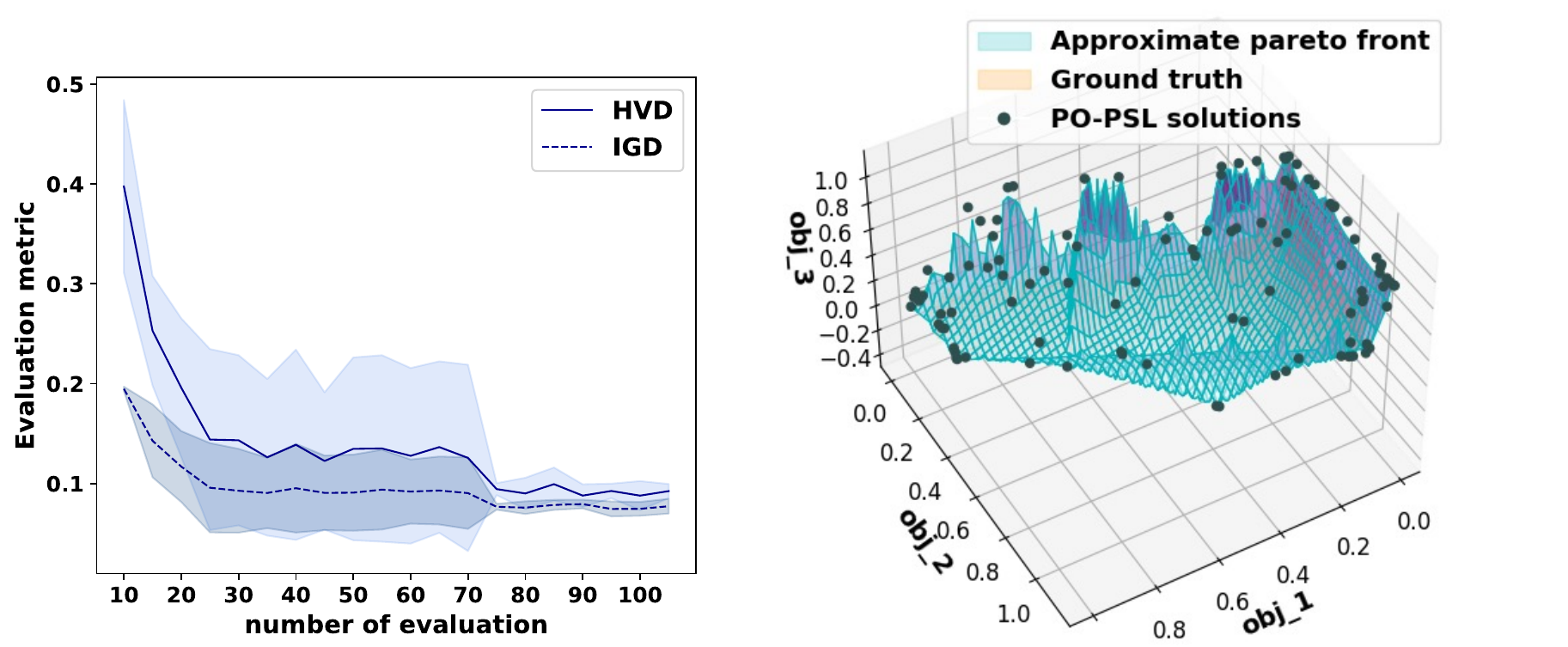}
 \caption{Results using the augmented Tchebycheff scalarization method. We conducted experiment using the RE5 data. Left: Sampling efficiency is evaluated in terms of HVD and IGD. Right: Pareto front approximation by PO-PSL.}
 \label{atch}
 \end{figure}
 \section{Conclusion}
We have proposed a novel Pareto set learning (PSL) algorithm that uses optimization-based modeling. Our formulation is a bilevel optimization algorithm and allows flexible exploration of the entire Pareto font. 
We have experimentally demonstrated that our proposed method significantly outperforms the existing methods on PSL in terms of sampling efficiency, computational efficiency, and accuracy in approximating the entire Pareto set. In particular, for complex optimization problems where the existing methods can only make partial predictions and struggle to predict at the boundary points, our method can accurately predict the entire Pareto set. 
 \section{Approximation and Pareto Optimality}
In this paper, we are interested in approximating the whole Pareto set/font of a MOO~\eqref{eq:moo}.
For real-world MOO problems, the analytical forms of the objective functions are often not known in advance. Hence, our goal is to have a good approximation of (\ref{eq:moo}) and obtain an $\epsilon$-approximate Pareto set/font for any arbitrarily small $\epsilon>0$ as stated in Theorem~\ref{theorem-ee-psl-pareto-opt}. Similar approximation and Pareto optimality guarantee have been provided for PSL of multiobjective neural combinatorial optimization~\citep{lin2022-PSL-MOCO}.\looseness=-1 
 \begin{theorem}\label{theorem-ee-psl-pareto-opt}
If our proposed method can generate an $\epsilon$-approximate solution $x \prec_\epsilon x^\ast$, where $x^\ast$ is the optimal solution of (\ref{eq:moo}), then it is able to generate an $\epsilon$-approximate Pareto set $\mc{P}_\epsilon$ to the corresponding MOO problem.
 \end{theorem}
 For the proof of Theorem~\ref{theorem-ee-psl-pareto-opt}, please see~\cite{lin2022-PSL-MOCO}. 
 \section{Limitation and Future Work} 
 The approximation guarantee made in Theorem~\ref{theorem-ee-psl-pareto-opt} depends on the approximation ability of the proposed algorithm. However, recent theoretical advancement on the convergence guarantee of bilevel optimization is quite promising~\citep{ji2021bilevel,fu2023convergence}. More specifically the author in \cite{fu2023convergence} provided sublinear regret bound for bilevel Bayesian optimization, suggesting convergence of both upper-level and lower-level parameters of bilevel optimization. Considering the recent advancement in theoretical convergence of bilevel optimization we are optimistic to design a more efficient algorithm for PO-PSL with rigorous theoretical guarantee which we consider as a potential future direction of research.

\bibliographystyle{plainnat}
\bibliography{myref}


\end{document}